\documentclass{article}

\PassOptionsToPackage{numbers,sort&compress}{natbib}

\usepackage[preprint]{neurips_2026}

\usepackage[utf8]{inputenc}
\usepackage[T1]{fontenc}

\usepackage{amsmath,amssymb}
\usepackage{mathtools}
\usepackage{bm}
\usepackage{mathdots}

\usepackage{xcolor}
\usepackage{comment}
\usepackage{enumitem}
\usepackage{microtype}
\usepackage{subcaption}

\usepackage{tikz}
\usepackage{pgfplots}
\pgfplotsset{compat=1.18}
\usetikzlibrary{calc}

\usetikzlibrary{matrix,fit,backgrounds,positioning}
\usepgfplotslibrary{groupplots}

\usepackage{hyperref}
\usepackage{cleveref}

\newcommand{\red}[1]{{\color{black}#1}}
\newcommand{\blue}[1]{{\color{black}#1}}

\newtheorem{theorem}{Theorem}
\newtheorem{proposition}{Proposition}

\newtheorem{construction}[theorem]{Construction}
\newtheorem{remark}[theorem]{Remark}
\newtheorem{definition}[theorem]{Definition}

\newcommand{\Pbb}{\mathbb{P}}

\newcommand{\mt}[1]{\textbf{[MT: #1]}}

\newcommand{\ee}{\mathbb{E}}
\newcommand{\pp}{\mathbb{P}}

\newcommand{\DP}[1]{{[\textcolor{blue}{DP: #1}]}}
\newcommand{\mc}[1]{\mathcal{#1}}
\newcommand{\mb}[1]{\mathbb{#1}}

\newcommand{\beq}{\begin{equation}}
\newcommand{\eeq}{\end{equation}}

\newcommand{\axisbreak}[1]{%
  \draw[, white, line width=3pt]
    ([yshift=5pt]#1) -- ([yshift=-5pt]#1);
  \draw[]
    ([xshift=-3pt, yshift=3pt]#1) -- ([xshift=3pt, yshift=-3pt]#1);
  \draw[]
    ([xshift=-3pt, yshift=-1pt]#1) -- ([xshift=3pt, yshift=-7pt]#1);
}
\newcommand{\barbreak}[1]{%
  \node at (#1) {\tiny$\vdots$};
}

\title{Preferent Compression Bounds Are Tight}

\author{%
  Dario Paccagnan$^{1}$ \qquad Marius Tirlea \\[2mm]
  {\small $^{1}$Department of Computing, Imperial College London, London, UK} \\
  {\small \texttt{d.paccagnan@imperial.ac.uk}}
}

\begin{document}
\maketitle

\begin{abstract}
The lack of rigorous safety and performance certificates remains a key bottleneck to the deployment of modern learning-based methods. Sample compression has recently emerged as a powerful tool for deriving such certificates, with particularly sharp bounds available for algorithms satisfying a so-called \emph{preference} property -- also known as stability in the learning theory literature. 
These bounds find direct application across domains as different as the Scenario Approach, Pick-to-Learn, and Support Vector methods. However, whether they are tight has remained an open problem. In this paper we resolve this question affirmatively and show that the state-of-the-art bound for preferent compressions is \emph{provably tight}. We establish this by exhibiting an explicit construction based on the uniform distribution and order statistics that attains the bound in the limit. Along the way, we also provide a considerably shorter and more accessible proof of this bound, requiring only elementary counting arguments and no infinite-dimensional duality.
\end{abstract}

\section{Introduction}

Across control, machine learning, and verification, a central challenge is to equip data-driven methods with \emph{risk certificates} against new unseen inputs. Data-driven controllers are required, for example, to ensure that a system remains within a safe region at all times. Similarly, binary classifiers used for medical diagnosis demand strict guarantees on misclassification rates. Indeed, safety is widely recognized as a key \mbox{bottleneck to the deployment of autonomy.}
\medskip

Against this backdrop, \emph{sample compression} has emerged as a tool to provide some of the sharpest risk certificates across all the above-mentioned domains, see, for example, \cite{paccagnan2023pick} for machine learning and \cite{paccagnan2025pick} for data-driven control. 
The central idea, dating back to Littlestone and Warmuth \cite{littlestone1986}, is that many algorithms naturally act by \emph{compressing} the data sample, or can be transformed to do so. More precisely, a data-driven algorithm is a sample compression if its output, e.g., a learned neural network or control policy, can be reconstructed from a \emph{subset} of the original dataset. This structural property enables the derivation of concentration inequalities which, in turn, yield risk certificates on a variety of properties of interest including constraint violation probability, and misclassification rates.

In particular, sharp concentration inequalities have been recently obtained for sample compression algorithms that exhibit an additional property referred to first as \emph{stability} in the learning community \cite{bousquet2020proper}, and subsequently as \emph{preference} \cite{campi2023compression}. Informally, this property requires that if the output produced on dataset $D$ can be reconstructed from $T \subseteq D$, then the same output must be produced on any dataset $V$ with $T \subseteq V \subseteq D$.

\medskip

While of interest in their own right, concentration results for sample compression satisfying the preference property -- henceforth \emph{preferent compressions} -- underpin key results across at least three important domains. 

First, in the \emph{Scenario Approach}, one solves an optimization problem subject to constraints imposed by data and seeks guarantees on the probability of violating a fresh, unseen constraint. In these settings, the sharpest available a posteriori results~\cite{campi2018wait} can be derived directly from concentration results for preferent compressions,  since the set of \emph{support constraints} -- those whose removal alters the optimum -- is sufficient to fully reconstruct the solution.

Second, in \emph{Pick-to-Learn} (P2L), one is given a black-box 
learning algorithm, e.g., SGD for neural network training, and seeks 
probabilistic guarantees on a generic property of interest. The key 
idea that makes this possible is that P2L transforms any such 
algorithm into a preferent compression, for which satisfaction of 
the property of interest can be certified directly via the probability 
of change of compression \cite{paccagnan2023pick}. This enables the deployment of concentration 
results for preferent compressions to arbitrary models and properties, 
including deep neural networks, Gaussian processes, and data-driven control policies~\cite{pmlr-v258-marks25a, paccagnan2025pick,paccagnan2025pickCDC}.

Third, in \emph{Support Vector methods}, such as Support Vector Machines and Support Vector Regression, one obtains classifiers and regressors through the solution to specially structured convex optimization problems. In this context the so-called \emph{support vectors} are sufficient to fully reconstruct an optimal solution \cite{cortes1995support}. This structure enables the deployment of concentration inequalities for preferent compressions, which yield sharp risk certificates for these methods \cite{hanneke2021stable}.

\medskip

Given the key role of concentration inequalities for preferent 
compressions, and the recent improvements of~\cite{bousquet2020proper}, 
\cite{hanneke2021stable}, and~\cite{campi2023compression}, a natural question arises:

\smallskip
\begin{center}
\emph{Are the state-of-the-art preferent compression bounds tight?}
\end{center}
\smallskip

In this paper we resolve this question affirmatively. Specifically, we show that the upper bound of ~\cite[\red{Thm 4}]{campi2023compression} is tight and cannot  be further improved.
\medskip

Two technical contributions underpin this result.
First, we provide a simpler and much shorter proof of this \red{upper bound, inspired by the original proof-strategy but, crucially, avoiding the need to formulate and solve an infinite-dimensional program over a space of measures.
Indeed, our proof requires only elementary counting arguments and is therefore accessible to a broader audience.} 
We believe this may be of independent interest.
Second, building upon this simplified proof, we provide an explicit construction consisting of a data-generating distribution and a preferent compression attaining tightness of the bound. 
Remarkably, this is achieved by a simple distribution, namely the uniform distribution, and correspondingly simple preferent compression, namely carefully-arranged order statistics. 

%

%
\medskip\noindent
\textbf{Organization.} The remainder of this paper is organized  as follows. \Cref{sec:preliminaries} introduces the necessary background on compression functions, the preference property, and the probability of change of compression. \Cref{sec:main-result} states the main result and presents the tightness construction together with a proof sketch. \Cref{sec:proof} contains the proof of the main result, beginning with a simplified proof of \Cref{prop:CG}.

\section{Preliminaries: Compression, Preference, and Risk Certificates}
\label{sec:preliminaries}
In this section we first introduce compression maps and the preference property. We then define the central quantity of interest: the \emph{probability of change of compression}. This quantity is key to risk certification, as an upper bound on the probability of change of compression translates directly into risk certificates across domains as different as the Scenario Approach, Pick-to-Learn, and Support Vector methods. We conclude the section by recalling the sharpest available bound on this quantity, which we will later show to be tight.
\medskip

Throughout the manuscript, we let $\mc{Z}$ denote the sample space and assume to have access to $N$ samples, which we collect in a dataset $D=\{z_1,\dots, z_N\}$.
We treat $D$ as a \emph{multiset} to account for possibly repeated observations. For brevity, we use standard set-theoretic terminology and notation such as subset, union, intersection, as well as $\cup, \cap, \setminus, \subseteq$, with the understanding that all such terms and symbols are interpreted in the multiset sense.\footnote{In a multiset, elements are not ordered but multiplicity is accounted for: $\{1,1,2\} = \{1,2,1\} \neq \{1,2\}$. 
The traditional set-based operations extend via the multiplicity function: given a multiset $U$, $\mu_U$ counts how many times each element of $\mc{Z}$ occurs in $U$. Then, $\mu_{U \cup U'}(z) = \mu_U(z)+\mu_{U'}(z)$, $\mu_{U \cap U'}(z) = \min\big\{\mu_U(z),\mu_{U'}(z)\big\}$, and $\mu_{U \setminus U'}(z) = \max\big\{0,\mu_U(z)-\mu_{U'}(z)\big\}$. Finally, $U \subseteq U'$ iff $\mu_U(z) \leq \mu_{U'}(z)$ for all $z$.} %
All multisets encountered have a finite number of elements and $|\cdot|$ denotes the cardinality, where each element is counted as many times as is its multiplicity. 
Finally, we model $D$ as a realization of independent and identically distributed random variables taking value in $\mc{Z}$ and defined over a probability space $(\Omega, \mc{F}, \mb{P})$.\footnote{We tacitly assume that all random quantities introduced are measurable, and reserve boldface to denote random quantities as opposed to their realizations, e.g., $\bm{D}$ as opposed to $D$.}
For instance, in supervised classification, $z$ is an input-label pair $z=(x,y)$ used to train a classifier. Similarly, in data-driven control, $z$ may represent process noise used to tune a control policy.

\medskip

\subsection{Compression maps and preference}
A number of learning-based algorithms act as \emph{compressions} of the dataset, in the sense that the final output depends only on a subset of the original data, or can be made to do so. This is typically formalized via two maps: a compression map, which selects the relevant subset, and a reconstruction map, which builds the learned model from it. Here we introduce only the former, since the tightness result we establish is independent of the reconstruction map. 

\medskip
\begin{definition}[Compression]\label{def:compress}
A compression function $c$ is a map that associates to any multiset of samples $U$ the subset $c(U)\subseteq U$. 
\end{definition}

We will be concerned with compression functions that satisfy a further property, first referred to as \emph{stability} in \cite{bousquet2020proper}, and later as  \emph{preference} in \cite{campi2023compression}. This property requests the compression function to map any multiset contained between $c(U)$ and $U$, into $c(U)$ itself.

\medskip
\begin{definition}[Preference]\label{def:pref}
A compression function $c$ satisfies the \emph{preference property} if for all multisets of samples~$U$
\[
c(U) = T  \quad\implies \quad 
c(V) = T \quad\text{for all} \quad T\subseteq V \subseteq U \, \, .
\]
\end{definition}
\medskip

Informally, $c$ cannot ``change its mind'': once it has selected $T$, restricting the data to any superset of $T$ contained in $U$ does not alter the selection, see also \Cref{fig:preference}. %
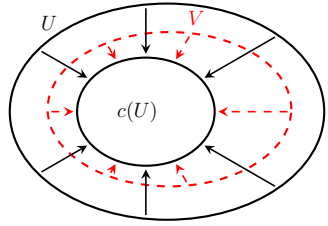
\begin{figure}[h!]
\centering
\begin{tikzpicture}[>=stealth, scale=0.65, transform shape]

\draw[black, line width=0.8pt] (0,0) ellipse (3.2cm and 2.2cm);

\draw[red, line width=0.8pt, dashed] (0.05,0.05) ellipse (2.48cm and 1.57cm);

\draw[black, line width=0.8pt] (-0.43,0) ellipse (1.40cm and 1.10cm);

\node[font=\large] at (-2.4, 1.8) {$U$};
\node[red, font=\large] at (0.6, 1.9) {$V$};
\node[font=\large] at (-0.6, 0) {$c(U)$};

\draw[->, black, line width=0.6pt] ( 2.199, 1.518) -- ( 0.763, 0.689);
\draw[->, black, line width=0.6pt] (-0.430, 2.105) -- (-0.430, 1.175);
\draw[->, black, line width=0.6pt] (-2.554, 1.226) -- (-1.623, 0.689);
\draw[->, black, line width=0.6pt] (-2.554,-1.226) -- (-1.623,-0.689);
\draw[->, black, line width=0.6pt] (-0.430,-2.105) -- (-0.430,-1.175);
\draw[->, black, line width=0.6pt] ( 2.199,-1.518) -- ( 0.763,-0.689);

\draw[->, red, dashed, line width=0.6pt] ( 2.454, 0.000) -- ( 1.045, 0.000);
\draw[->, red, dashed, line width=0.6pt] ( 0.453, 1.530) -- ( 0.186, 1.067);
\draw[->, red, dashed, line width=0.6pt] (-1.197, 1.328) -- (-1.046, 1.067);
\draw[->, red, dashed, line width=0.6pt] (-2.354, 0.000) -- (-1.905, 0.000);
\draw[->, red, dashed, line width=0.6pt] (-1.149,-1.246) -- (-1.046,-1.067);
\draw[->, red, dashed, line width=0.6pt] ( 0.399,-1.436) -- ( 0.186,-1.067);

\end{tikzpicture}%
\caption{A compression function $c$ is preferent if it maps any subset contained between $U$ and $c(U)$, e.g., $V$ in the figure, onto $c(U)$ itself.}
\label{fig:preference}
\vspace*{-1mm}
\end{figure}

\subsection{Probability of change of compression and existing bounds}
We are interested in bounding the probability of change of compression associated with the dataset $D$. This quantity measures the probability that adding a fresh sample $\bm{z}\sim\mb{P}$ to $c(D)$ changes the compression, i.e., $c(c(D) \cup \{\bm{z}\}) \neq c(D)$. This is formalized next. 

\medskip
\begin{definition}[Probability of change of compression]\label{def:phi}
Given a compression function $c$ and a dataset $D$, the probability of change of compression is given by
\begin{equation}\label{eq:phi}
\phi(D) = \mb{P}\big\{c(c(D) \cup \{\bm{z}\}) \neq c(D) \;|\; D\big\} \, \, ,
\end{equation}
where $\bm{z} \sim \Pbb$ is a fresh draw, independent of $D$.
\end{definition}

\medskip

In our model, $D$ is itself a realization of the random 
variable ${\bf D}$, so that $\phi({\bf D})$ 
is also a random variable. We are thus interested in bounding this quantity with high confidence with respect to the draw of $\bm{D}$. In particular, %
we seek to show that, with confidence $1-\delta$, 
\[
\phi(\bm{D}) \leq \varepsilon(|c(\bm{D})|, N, \delta)
\]
for a suitably defined function $\varepsilon$ depending solely on the cardinality of the compressed set $|c(D)|$, the number of samples $N$, and the confidence $\delta$.
\medskip

\red{
Building on the seminal work of Littlestone and Warmuth~\cite{littlestone1986}, a series of contributions~\cite{bousquet2020proper,hanneke2021stable} directly bound the risk of the reconstructed hypothesis (e.g., its misclassification rate) for compression schemes satisfying the preference property, therein referred to as stability. The more recent work of~\cite{campi2023compression}, instead, studies preferent compressions in isolation, %
 introducing the notion of probability of change of compression, and providing, amongst others, the following upper bound on this quantity.
}

\begin{proposition}[\hspace*{-1.03mm}{\cite[Thm 4]{campi2023compression}}]
\label{prop:CG}
Let $c$ be a preferent compression. %
Then, 
\beq
\mb{P}\{\phi(\bm{D}) \leq \varepsilon(|c(\bm{D})|, N, \delta)\}\ge 1-\delta
\label{eq:bound-change-comp}
\eeq
where $\varepsilon(k, N, \delta)$ is, for $k\le N-1$, the unique root of the following equation in $(0,1)$ 
\beq
\label{eq:calibration}
\frac{\binom{N}{k}(1-\varepsilon)^{N-k}}{\sum_{m = k}^{N-1} \binom{m}{k}(1-\varepsilon)^{m-k}} = \frac{\delta}{N}\,\, ,
\eeq
and $\varepsilon(N, N, \delta)=1$.
\end{proposition}

\section{Main result: Preferent compression bounds are tight}
\label{sec:main-result}

We are now ready to state our main result: we show that the bound on the probability of change of compression obtained in \Cref{prop:CG} is tight. More precisely, we will exhibit a distribution and a sequence of preferent compressions for which the confidence in \eqref{eq:bound-change-comp} is \mbox{\emph{arbitrarily close} to $1-\delta$.}%
\medskip

\begin{theorem}\label{thm:tight}
\vspace*{1mm}
The bound of \Cref{prop:CG} is tight. Specifically, for any $N \in \mb{N}$ 
and $\delta \in (0,1)$,
\[
\inf_{\mb{P},\, c \text{ preferent}}\, 
\mb{P}\Big\{ \phi(\bm{D}) \leq  \varepsilon(|c(\bm{D})|, N, \delta)\Big\} = 1- \delta.
\vspace*{2mm}\]
\end{theorem}
This result shows that under the sole preference property, one cannot hope to secure a better result than \Cref{prop:CG}.

\medskip

To prove the theorem, we exhibit a distribution $\mb{P}$ and a sequence of preferent compressions $\{c_\eta: \eta > 0\}$ for which the confidence in \eqref{eq:bound-change-comp} converges to exactly $1-\delta$ as $\eta \to 0$. Specifically, we consider the uniform distribution on $[0,1]$, with compressions $c_\eta$ defined as in the following construction, also presented in \Cref{fig:tight-construction}.

\begin{construction}\label{con:main} \rm
Fix $N \in \mb{N}$, $\delta \in (0,1)$, and ${\red{\eta \ge 0}}$. %
Define thresholds 
\[
t_k =  \varepsilon(k, \, N, \, \delta) + \eta \quad\text{for}\quad k = 0, \dots, N-1\,\,.
\]
For any finite multiset $U$, list all its elements (including repetitions) in non-decreasing order, breaking ties according to any rule. For instance, break ties randomly for samples that are identical. As we will see, when elements are identical, the order in which they appear in this list will be irrelevant for the purpose of constructing the final compression $c_\eta(U)$.
Define the resulting ordered list as

$$
u_{(1)} \leq u_{(2)} \leq \cdots \leq u_{(|U|)} \, \, ,
$$

\noindent
and let
\[
k^* = \max\big\{0 \leq k \leq |U|,~ k\leq N ~\text{such that}~u_{(j)} \leq t_{j-1} \text{ for all } 1 \leq j \leq k\big\} \, \, ,
\]
\noindent with the convention that $k^* = 0$ when $u_{(1)} > t_0$. Finally, define the compression $c_\eta$ as the multiset

\[
c_\eta(U) = \big\{u_{(1)}, \, \ldots, \, u_{(k^*)}\big\} \, \, ,
\]
\noindent where $c_\eta(U)=\emptyset$ when $k^* = 0$.

\begin{figure}
\centering\begin{tikzpicture}

\draw[line width=0.8pt] (0, 0) -- (8, 0);

\draw[line width=0.8pt] (0, -0.25) -- (0, 0.25);
\draw[line width=0.8pt] (8, -0.25) -- (8, 0.25);
\node[below, font=\small] at (0, 0.7) {$0$};
\node[below, font=\small] at (8, 0.7) {$1$};

\foreach \x/\name in {1.6/$t_0$, 2.8/$t_1$, 4.0/$t_2$, 4.96/$t_3$%
}{
    \draw[line width=0.6pt] (\x, 0) -- (\x, -0.18);
    \node[below, font=\small] at (\x, -0.18) {\name};
}
\draw[line width=0.6pt] (7.5, 0) -- (7.5, -0.18);
\node[below, font=\small] at (7.5, -0.18) {$t_{N-1}$};

\foreach \x/\name in {0.96/$u_{(1)}$, 2.24/$u_{(2)}$, 3.44/$u_{(3)}$, %
5.44/$u_{(4)}$}{
    \draw[line width=0.7pt] (\x-0.12, -0.12) -- (\x+0.12, 0.12);
    \draw[line width=0.7pt] (\x-0.12,  0.12) -- (\x+0.12,-0.12);
    \node[above, font=\small] at (\x, 0.18) {\name};
}

\node[font=\small] at (6.4, 0.35) {$\cdots$};
\node[font=\small] at (6.4, -0.35) {$\cdots$};
\end{tikzpicture}
\caption{The compression $c_\eta(U)$ lists all the elements of $U$ in non-decreasing order (including repetitions) as $u_{(1)}\leq u_{(2)}\leq \dots \leq u_{(|U|)}$. It then iterates over this list, starting from $u_{(1)}$, and checking whether each element $u_{(j)}$ falls below the corresponding threshold $t_{j-1} = \varepsilon(j-1, N, \delta) +\eta$, stopping at the first failure or upon reaching $N$. In this example $u_{(1)}$, $u_{(2)}$, $u_{(3)}$ pass the test; but not $u_{(4)}$ since $u_{(4)}>t_3$. Thus, independently of the values taken by $u_{(5)},~\dots~,u_{(|U|)}$ the compression returns $c_\eta(U) = \{u_{(1)},~u_{(2)},~u_{(3)}\}$. 
    }
\label{fig:tight-construction}
\end{figure}
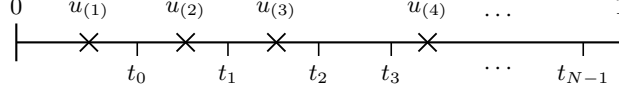
\medskip

Note that the multiset $c_\eta(U)$ is well defined and independent of the tie-breaking rule previously employed, should one have been necessary. The reason is twofold. First, $k^*$ is independent of the tie-breaking rule chosen. Indeed, if two elements $u_{(j)}$ and $u_{(j+1)}$ are identical, swapping them places the same value in position $j$ and $j+1$, which are compared against the fixed threshold $t_{j-1}$ and $t_j$, leaving $k^*$ unchanged. Second, recall that a multiset is an \emph{unordered} collection of (possibly repeated) elements. When, for example, two samples $u_{(j)}$ and $u_{(j+1)}$ are identical, the multiset is not able to distinguish whether we have \mbox{included $u_{(j)}$ or $u_{(j+1)}$.}%
\medskip

This construction places the elements of $U$ in monotonically increasing order (including repetitions), iterates through the list checking whether each element falls below its corresponding threshold, and stops at the first failure or upon reaching position $N$. The returned compressed multiset consists of all elements up to the stopping point.
\end{construction}

\begin{figure}[b!]
\centering
\begin{tikzpicture}[x=1mm, y=1mm, >=stealth]
\draw[->, line width=0.5pt] (2.752,0) -- (43.5,0);
\draw[->, line width=0.5pt] (2.752,0) -- (2.752,22.0);
\draw[->, line width=0.5pt] (2.825,0.111) -- (30.5,18.3);

\draw[line width=0.5pt] (1.566,0.000) -- (3.746,0.000);
\draw[line width=0.5pt] (17.846,10.700) -- (20.026,10.700);
\draw[line width=0.5pt] (19.442,11.749) -- (21.622,11.749);
\draw[line width=0.5pt] (23.987,14.736) -- (26.167,14.736);
\node[left, font=\scriptsize] at (23.8,14.7) {$1$};

\node[below left, font=\scriptsize] at (2.752,0) {$0$};

\foreach \x/\lab in {
    2.752/0,
    8.044/1,
    13.336/{\cdots},
    18.627/{\cdots},
    23.919/{k},
    29.211/{\cdots},
    34.502/{\cdots},
    39.794/{N}}{
    \draw[line width=0.5pt] (\x,-0.5) -- (\x,0.5);
    \node[below, font=\scriptsize] at (\x,-0.8) {$\lab$};
}

\draw[dashed, line width=0.5pt] (25,14.736) -- (62.5,14.736);

\draw[red, line width=1pt] plot[smooth] coordinates {
    ( 5.924,  2.091)
    (16.931,  5.769)
    (25.181,  7.903)
    (33.189,  9.371)
    (39.979, 10.503)
    (47.479, 11.937)
    (54.383, 13.001)
    (62.395, 14.731)
};

\filldraw[red] (62.114,14.684) circle (0.8pt);

\draw[dashed, line width=0.5pt] (39.794,0) -- (62.395,14.731);

\draw[dotted, line width=0.5pt] (18.852,10.7) -- (40,10.7);
\draw[dotted, line width=0.5pt] (20.722,11.749) -- (41.702,11.749);
\node[left, font=\scriptsize] at (20.5,12.8) {$\varepsilon_k{+}\eta$};
\node[left, font=\scriptsize] at (18.5,10.522) {$\varepsilon_k$};

\draw[dotted, line width=0.5pt] (23.919,0.000) -- (41.702,11.664);

\filldraw[red]  ( 6.005, 2.286) circle (0.8pt);
\filldraw[blue] ( 7.567, 3.294) circle (0.8pt);
\draw[->, blue, line width=1pt] ( 7.567, 3.294) -- ( 7.567, 5.944);
\filldraw[red]  (17.039, 5.886) circle (0.8pt);
\filldraw[blue] (18.602, 6.894) circle (0.8pt);
\draw[->, blue, line width=1pt] (18.602, 6.894) -- (18.602, 9.720);
\filldraw[red]  (25.348, 8.059) circle (0.8pt);
\filldraw[blue] (26.910, 9.067) circle (0.8pt);
\draw[->, blue, line width=1pt] (26.910, 9.067) -- (26.910,13.483);
\filldraw[red]  (32.969, 9.407) circle (0.8pt);
\filldraw[blue] (34.531,10.415) circle (0.8pt);
\draw[->, blue, line width=1pt] (34.531,10.415) -- (34.531,17.570);
\filldraw[red]  (40.140,10.655) circle (0.8pt);
\filldraw[blue] (41.702,11.664) circle (0.8pt);
\draw[->, blue, line width=1pt] (41.702,11.664) -- (41.702,19.790);
\filldraw[red]  (47.524,11.978) circle (0.8pt);
\filldraw[blue] (49.086,12.987) circle (0.8pt);
\draw[->, blue, line width=1pt] (49.086,12.987) -- (49.086,17.933);
\filldraw[red]  (54.489,13.101) circle (0.8pt);
\filldraw[blue] (56.051,14.110) circle (0.8pt);
\draw[->, blue, line width=1pt] (56.051,14.110) -- (56.051,16.759);
\end{tikzpicture}
\caption{Construction~\ref{con:main} is such that for every $0\leq k\leq N-1$, when $|c_\eta(\bm{D})|=k$, the probability of change of compression equals exactly $\varepsilon_k+\eta$ (represented as blue point masses), sitting exactly $\eta$ above the bound $\varepsilon_k$ (red curve), and thus strictly violating 
it. When $|c_\eta(\bm{D})|=N$, the bound equals $1$ and is thus never violated.}
\label{fig:construction-with-diracs}
\end{figure}
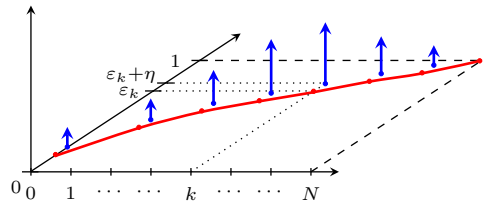

We now outline the key ideas underlying the proof, deferring the details to Section~\ref{sec:proof}.
\medskip

\textbf{Proof Sketch.} 
Three main ideas are \red{woven} into the proof of \Cref{thm:tight}. 
First, we verify that each $c_\eta$ is preferent. Informally, this 
holds because the compression is constructed greedily from left to 
right: a sample is retained if and only if it falls below the corresponding threshold and all other samples on the left of it do so. Whether a new sample triggers this condition does not depend on whether other non-compressed samples are present. Hence adding any point from $U \setminus c_\eta(U)$ to $c_\eta(U)$ does not alter the compression, which is precisely the preference property.

Second, the compressions $\{c_\eta: \eta > 0\}$ have a special 
structure: all datasets for which the compression has size strictly less than $N$ violate the bound given by $\varepsilon$, while datasets with compression size exactly $N$ trivially satisfy it (the bound equals $1$ in that case). More precisely, we will show that, when  $|c_\eta(\mathbf{D})| < N$, the probability of change of compression is a point mass strictly above the bound $\varepsilon(k, N, \delta)$, see \Cref{fig:construction-with-diracs}. The total probability of violating the bound therefore reduces to the probability of obtaining a compression 
of size strictly less than $N$.

Finally, the thresholds are chosen so that this probability 
converges to exactly $\delta$ as $\eta \to 0^+$. This is the heart of 
the argument, and relies on the defining equation of $\varepsilon(k, 
N, \delta)$ in an essential way. To aid intuition, we provide 
numerical evidence for these claims in \Cref{fig:convergence} 
and \Cref{fig:histogram} before turning to the proof.

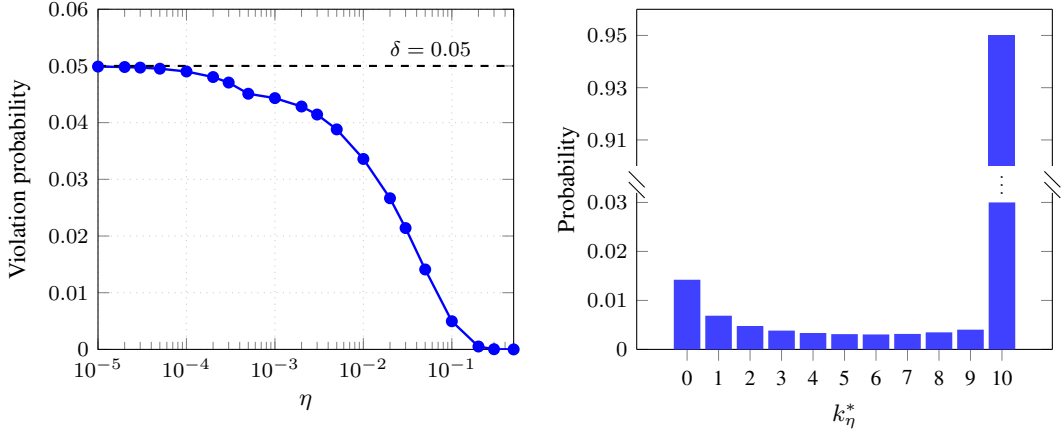
\begin{figure}[t!]
\centering

\begin{subfigure}[b]{0.48\linewidth}
    \centering
    \begin{tikzpicture}

\begin{axis}[
  name=leftaxis,
  width=0.82\linewidth,
  height=4.5cm,
  scale only axis,
  xmode=log,
  scaled y ticks=false,
  xmin=1e-5,
  xmax=5e-1,
  ymin=0,
  ymax=0.06,
  xlabel={$\eta$},
  ylabel={Violation probability},
  xlabel style={font=\small},
  ylabel style={font=\small},
  tick label style={font=\footnotesize},
  grid=major,
  grid style={dotted, gray!40},
  ytick={0,0.01,0.02,0.03,0.04,0.05,0.06},
  yticklabels={$0$,$0.01$,$0.02$,$0.03$,$0.04$,$0.05$,$0.06$},
]

\addplot[
  blue,
  mark=*,
  mark size=1.8pt,
  line width=0.9pt
] coordinates {
   (5e-1, 8.00000000e-09)
   (3e-1, 3.38400000e-05)
   (2e-1, 0.000509045)
   (1e-1, 0.004966931)
   (5e-2, 0.014099261)
   (3e-2, 0.021401305)
   (2e-2, 0.026661785)
   (1e-2, 0.033592605)
   (5e-3, 0.038808019)
   (3e-3, 0.041428940)
   (2e-3, 0.042844629)
   (1e-3, 0.044323190)
   (5e-4, 0.045096821)
   (3e-4, 0.047068335)
   (2e-4, 0.048041243)
   (1e-4, 0.049022784)
   (5e-5, 0.049512587)
   (3e-5, 0.049706314)
   (2e-5, 0.049809806)
   (1e-5, 0.049885468)
};

\addplot[
  dashed,
  black,
  line width=0.8pt,
  domain=1e-5:5e-1
] {0.05}
node[
  pos=0.8,
  above,
  font=\footnotesize
] {$\delta=0.05$};

\end{axis}

\coordinate (natwestL) at (current bounding box.west);
\coordinate (nateastL) at (current bounding box.east);

\coordinate (frametopL) at ([yshift=8pt]leftaxis.north);
\coordinate (framebottomL) at ([yshift=-38pt]leftaxis.south);

\pgfresetboundingbox

\path[use as bounding box]
  (natwestL |- framebottomL)
  rectangle
  (nateastL |- frametopL);

\end{tikzpicture}\par
    \caption{Violation probability as a function of $\eta$.}%
    \label{fig:convergence}
\end{subfigure}%
\hfill%
\begin{subfigure}[b]{0.48\linewidth}
    \centering
    \begin{tikzpicture}

\pgfplotsset{
  every axis/.append style={
    width=0.82\linewidth,
    scale only axis,
    ybar,
    bar width=10pt,
    enlarge x limits=0.1,
    xmin=0.5,
    xmax=11.5,
    xlabel style={font=\small},
    ylabel style={font=\small},
    tick label style={font=\footnotesize},
  }
}

\begin{groupplot}[
  group style={
    group size=1 by 2,
    vertical sep=10pt,
  },
]

\nextgroupplot[
  height=2.074cm,
  ymin=0.90,
  ymax=0.96,
  ytick={0.91,0.93,0.95},
  axis x line*=top,
  xtick=\empty,
]

\addplot[
  fill=blue!75!white,
  draw=none
] coordinates {
  (1,nan)
  (2,nan)
  (3,nan)
  (4,nan)
  (5,nan)
  (6,nan)
  (7,nan)
  (8,nan)
  (9,nan)
  (10,nan)
  (11,0.950114532)
};

\nextgroupplot[
  height=2.074cm,
  ymin=0,
  ymax=0.032,
  ytick={0,0.01,0.02,0.03},
  scaled y ticks=false,
  yticklabel style={
    /pgf/number format/fixed,
    font=\footnotesize
  },
  axis x line*=bottom,
  xtick={1,2,3,4,5,6,7,8,9,10,11},
  xticklabels={0,1,2,3,4,5,6,7,8,9,10},
  xlabel={$k_\eta^*$},
]

\addplot[
  fill=blue!75!white,
  draw=none
] coordinates {
  (1,0.014217931)
  (2,0.00687626)
  (3,0.004776847)
  (4,0.003834388)
  (5,0.00334977)
  (6,0.003114561)
  (7,0.003049537)
  (8,0.00315727)
  (9,0.003480478)
  (10,0.004028426)
  (11,0.030)
};

\coordinate (bin11top) at (axis cs:11,0.016);

\end{groupplot}

\coordinate (ylabelcenter) at
  ($(group c1r1.north west)!0.5!(group c1r2.south west)$);

\node[
  rotate=90,
  font=\small
] at ([xshift=-9mm]ylabelcenter)
{Probability};

\path
  (group c1r1.south west) --
  (group c1r2.north west)
  coordinate[midway] (leftbreak);

\path
  (group c1r1.south east) --
  (group c1r2.north east)
  coordinate[midway] (rightbreak);

\path
  (bin11top |- leftbreak)
  coordinate (barbreak);

\axisbreak{leftbreak}
\axisbreak{rightbreak}
\barbreak{barbreak}

\coordinate (natwestR) at (current bounding box.west);
\coordinate (nateastR) at (current bounding box.east);

\coordinate (frametopR) at
  ([yshift=8pt]group c1r1.north);

\coordinate (framebottomR) at
  ([yshift=-38pt]group c1r2.south);

\pgfresetboundingbox

\path[use as bounding box]
  (natwestR |- framebottomR)
  rectangle
  (nateastR |- frametopR);

\end{tikzpicture}\par
    \caption{Distribution of the compression cardinality $\bm{k}_\eta^*$.}%
    \label{fig:histogram}
\end{subfigure}

\caption{%
Illustration of Construction~\ref{con:main} for $N=10$ and
$\delta=0.05$.
(a) The violation probability
$\mathbb{P}\{\phi(\mathbf{D}) >
\varepsilon(|c(\mathbf{D})|,N,\delta)\}$
converges to $\delta$ from below as $\eta\to0^+$, as per
\Cref{thm:tight}.
(b) Distribution of $\bm{k}_\eta^*=|c_\eta(\bm{D})|$ for
$\eta=10^{-5}$. Note that the y-axis is broken to accommodate the height of the last bin. In accordance with the discussion above, as $\eta\to0^+$, $1-\delta=0.95$ of the probability mass concentrates at $k_\eta^*=10$, while the remaining
$\delta=0.05$ is spread across $k_\eta^*\in\{0,\ldots,9\}$.
Both plots were generated by sampling $10^9$ datasets.}
\label{fig:construction-numerics}
\end{figure}

\section{Proof of the main result}
\label{sec:proof}
The proof of Theorem~\ref{thm:tight} builds on a careful examination of the bound in Proposition~\ref{prop:CG}. To this end, we begin by providing a simplified proof of Proposition~\ref{prop:CG} itself. We believe this may be of independent interest, as it is considerably shorter than the original, requires only elementary counting arguments and no infinite-dimensional duality, and may therefore be more accessible. Crucially, it also exposes the counting structure underlying the bound, which is precisely what the tightness construction exploits. The proof of Theorem~\ref{thm:tight} then follows naturally.

\subsection{Simplified Proof of \Cref{prop:CG}}
\label{subsec:simplifiedproof}

\red{
We first prove the result under the following non-concentrated mass assumption, and then show how to lift this requirement in the Appendix:
\begin{equation}
    \mathbb{P}\{\bm{z}_i=z\}=0
    \qquad \text{for all } z\in\mathcal Z.
    \label{eq:non-concentrated-mass}
\end{equation}
That is, we assume, for now, that $\mathbb{P}$ has no atoms.
}
\medskip

For any $k=0,\dots,N$, let $D_k = \{z_1,\dots,z_k\}$, with $D_0=\emptyset$, to highlight the cardinality of the dataset. 
\medskip

We start from the quantity we wish to bound.  Since $|c(D_N)|$ 
ranges over $0, \dots, N$, we can decompose over the mutually exclusive events $|c(\bm{D}_N)|=k$ as follows:
\[
\begin{split}
\mathbb{P}\{ \phi(\bm{D}_N) > \varepsilon(|c(\bm{D}_N)|, N, \delta)\}
&= 
\sum_{k=0}^{N} \mathbb{P}\{ |c(\bm{D}_N)|=k~\text{and}~\phi(\bm{D}_N) > \varepsilon(k, N, \delta)\}\\
&= 
 \sum_{k=0}^{N} \binom{N}{k}\mathbb{P}\{ c(\bm{D}_N)=\bm{D}_k~\text{and}~\phi(\bm{D}_N) > \varepsilon(k, N, \delta)\}\,\,.
\end{split}
\]

\red{The last equality holds for two reasons. First, no two samples coincide almost surely, owing to \eqref{eq:non-concentrated-mass}, so $c(\bm{D}_N)$ equals exactly one of the $\binom{N}{k}$ sub-multisets of $\bm{D}_N$ of cardinality $k$. These events are therefore almost surely
disjoint and partition $\{|c(\bm{D}_N)|=k\}$.} Second, since the samples are i.i.d., their joint distribution is exchangeable: it is invariant under any permutation of the indices $1, \dots, N$. As a consequence, for any two subsets of cardinality $k$, the probability that $c(\bm{D}_N)$ equals one such subset is the same as the probability that it equals another. Since there are $\binom{N}{k}$ such subsets, the probability of the event $\{|c(\bm{D}_N)|=k~\text{and}~\phi(\bm{D}_N) > \varepsilon(k,N,\delta)\}$ equals $\binom{N}{k}$ times the probability of the specific event $\{c(\bm{D}_N)=\bm{D}_k~\text{and}~\phi(\bm{D}_N) > \varepsilon(k,N,\delta)\}$.
\medskip

\red{
Up to this point, the proof follows the same general approach as that of~\cite[Thm 4]{campi2023compression}. It is with the forthcoming definition of $p_{k,m}$ that the two depart. The authors of \cite{campi2023compression} introduce, in place of a scalar, a suitable measure, derive conditions that these measures must satisfy, and formulate an infinite-dimensional linear program whose value is evaluated by truncation and dualization. Instead, we show that the quantities $p_{k,m}$ satisfy two simple relations on a triangular array, and that these are sufficient to \mbox{derive the same bound directly.}}
\medskip

Now introduce $p_{k,m}$ for any $0\le k\le m \le N$ as 

\[p_{k,m} := \mathbb{P}\{c(\bm{D}_m)=\bm{D}_k~\text{and}~\phi({\bm{D}_m}) >\varepsilon(k, N, \delta)\}. \] 

Each quantity $p_{k,m}\in[0,1]$ represents the probability that compressing $\bm{D}_m$ yields $\bm{D}_k$ and that jointly $\phi(\bm{D}_m) > \varepsilon(k, N, \delta)$. The quantities 
\[\{p_{k,m}~:~0\le k\le m \le N\}\] 
form a triangular array, visualized in \Cref{fig:triangular-array}, and will play a central role in what follows. 

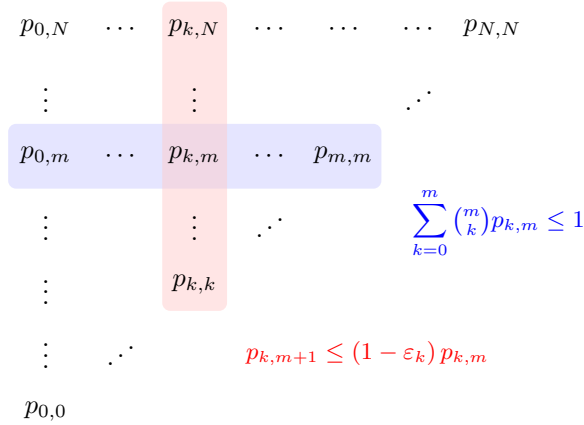
\begin{figure}[t!]
\centering
\begin{tikzpicture}
\matrix (M) [
    matrix of math nodes,
    row sep=0.0em,
    column sep=0.0em,
    nodes={anchor=center, inner sep=3pt, minimum width=2.8em, minimum height=2.4em},
] {
    p_{0,N} & \cdots  & p_{k,N}  & \cdots  & \cdots   & \cdots   & p_{N,N} \\
    \vdots  &         & \vdots   &         &          & \iddots  &         \\
    p_{0,m} & \cdots  & p_{k,m}  & \cdots  & p_{m,m}  &          &         \\
    \vdots  &         & \vdots   & \iddots &          &          &         \\
    \vdots  &         & p_{k,k}  &         &          &          &         \\
    \vdots  & \iddots &          &         &          &          &         \\
    p_{0,0} &         &          &         &          &          &         \\
};

\begin{scope}[on background layer]
    \node[fill=blue!20, fill opacity=0.5, rounded corners=3pt,
          inner sep=0pt,
          fit=(M-3-1)(M-3-5)] {};
\end{scope}

\begin{scope}[on background layer]
    \node[fill=red!20, fill opacity=0.5, rounded corners=3pt,
          inner sep=-2pt,
          fit=(M-1-3)(M-5-3)] {};
\end{scope}

\node[anchor=west, font=\small, text=blue]
    at ([yshift=-2.5em, xshift=0.8em]M-3-5.east)
    {$\displaystyle\sum_{k=0}^{m}\tbinom{m}{k}p_{k,m} \leq 1$};

\node[anchor=north, font=\small, text=red]
    at ([yshift=-0.6em, xshift=6.5em]M-5-3.south)
    {$p_{k,m+1} \leq (1-\varepsilon_k)\,p_{k,m}$};
\end{tikzpicture}
\caption{Visualization of the triangular array $\{p_{k,m} : 0 \leq k \leq m \leq N\}$. The first relation, reported in blue, controls the weighted sum along each row; the second relation, reported in red, links adjacent entries within each column. Thanks to \eqref{eq:p-change-comp-equal-p-terms}, it suffices to bound $\sum_{k=0}^{N-1} \binom{N}{k} p_{k,N}$, which is a linear combination of the terms in the top row of the triangular array.}
\label{fig:triangular-array}
\end{figure}

We can thus write 
\beq 
\mathbb{P}\{ \phi(\bm{D}_N) > \varepsilon(|c(\bm{D}_N)|, N, \delta)\}= \sum_{k=0}^{N-1} \binom{N}{k} p_{k,N}, 
\label{eq:p-change-comp-equal-p-terms} 
\eeq

\noindent where the $k=N$ term vanishes since $\varepsilon(N,N,\delta)=1$ and $\phi(\bm{D}_N)\le 1$ (being a conditional probability), so the event $\{\phi(\bm{D}_N) > \varepsilon(N,N,\delta)\}$ has probability zero.

\medskip
Two key relations govern the triangular array.

\noindent\textbf{First relation.} For any $0\le m \le N$, 
\[ \sum_{k=0}^{m} \binom{m}{k} p_{k,m} \leq 1\,\,. \] 
This holds because, applying the same \red{disjointness} and symmetry argument as above to $\bm{D}_m$ in place of $\bm{D}_N$, the left hand side equals $\mathbb{P}\{\phi(\bm{D}_m) > \varepsilon(|c(\bm{D}_m)|,N,\delta)\}$, which -- being a probability -- is at most one. This relation links together entries within the same row $m$ of the triangular array, see \Cref{fig:triangular-array}.

\medskip
\noindent\textbf{Second relation.} By the preference property,
\[
p_{k,m+1} \leq (1-\varepsilon_k)\, p_{k,m} 
\qquad \text{for all } 0 \leq k \leq m \leq N-1,
\]
where $\varepsilon_k := \varepsilon(k,N,\delta)$. This links 
together entries \emph{within the same column} of the array, see \Cref{fig:triangular-array}. Its proof is deferred to the Appendix.

\medskip
\noindent\textbf{Proof of the bound.} Thanks to \eqref{eq:p-change-comp-equal-p-terms}, it suffices to show that $\sum_{k=0}^{N-1}\binom{N}{k}p_{k,N}\leq \delta$. Note that the quantity to bound consists of a linear combination of the terms appearing in the top row of the triangular array in \Cref{fig:triangular-array}. We do so employing the two relations just introduced. The second relation allows us to descend from level $N$ to any lower level $m < N$ in the triangular array, that is, to bound $p_{k,N}$ with
\[
p_{k,N} \leq (1-\varepsilon_k)^{N-m} p_{k,m}\,\,.
\]

\noindent 
Rather than choosing to bound $p_{k,N}$ using a single level $m$, we  combine \emph{all} levels simultaneously. Specifically, for each column $k$, introduce weights $\alpha_{k,m} \geq 0$ with $\sum_{m=k}^{N-1} \alpha_{k,m} = 1$, and combine the above inequalities
\[
p_{k,N} \leq \sum_{m=k}^{N-1}(1-\varepsilon_k)^{N-m}\alpha_{k,m} \, p_{k,m}\,\,.
\]

\noindent
Substituting into \eqref{eq:p-change-comp-equal-p-terms} and swapping the order of summation:
\[
\mathbb{P}\{ \phi(\bm{D}_N) > \varepsilon(|c(\bm{D}_N)|, N, \delta)\}= 
\sum_{k=0}^{N-1}\binom{N}{k} p_{k,N} 
\leq 
\sum_{m=0}^{N-1}\sum_{k=0}^{m}
{\binom{N}{k}(1-\varepsilon_k)^{N-m}\alpha_{k,m}}
\, p_{k,m}\,\,.
\]

We now wish to apply the first relation, by which  $\sum_{k=0}^{m}\binom{m}{k}p_{k,m}\leq 1$ along each row $m$ of the triangular array. In order to do so, we impose
\[
\binom{N}{k}
(1-\varepsilon_k)^{N-m}\alpha_{k,m} 
=
\binom{m}{k}
\delta_m
\]
for some given $\delta_m$ independent of $k$. This constraint implies that $\alpha_{k,m}$ satisfies
\beq
\alpha_{k,m} = \delta_m\frac{\binom{m}{k}}{\binom{N}{k}}
(1-\varepsilon_k)^{m-N}\,\,.
\label{eq:alpha-k-m}
\eeq
At the same time, $\sum_{m=k}^{N-1} \alpha_{k,m} = 1$ imposes a consistency condition on $\varepsilon_k$. Indeed, substituting the expression for $\alpha_{k,m}$ from \eqref{eq:alpha-k-m} into $\sum_{m=k}^{N-1} \alpha_{k,m} = 1$ and simplifying gives

\[
\frac{\binom{N}{k}(1-\varepsilon_k)^{N-k}}
{\sum_{m=k}^{N-1}\binom{m}{k}(1-\varepsilon_k)^{m-k}} = 
\delta_m \,\,,
\]
which is precisely the calibration equation~\eqref{eq:calibration} upon 
setting $\delta_m = \frac{\delta}{N}$. With this choice, 
the result follows
\[
\mathbb{P}\{ \phi(\bm{D}_N) > \varepsilon(|c(\bm{D}_N)|, N, \delta)\} 
\leq 
\sum_{m=0}^{N-1}\delta_m \sum_{k=0}^{m}
\binom{m}{k}p_{k,m}
\leq\sum_{m=0}^{N-1} \delta_m = \delta\,\,.
\]

\subsection{Proof of \Cref{thm:tight}}
We structure the proof in three parts:

\begin{enumerate}

\item[(i)] We show that $c_\eta$ satisfies the preference property for any choice of $\red{\eta \ge 0}$.

\item[(ii)] We consider the probability of change of compression, $\phi_\eta(D)$, associated to $c_\eta$ and the chosen distribution $\mb{P}$.  We then show that $\mb{P}\Big\{ \phi_\eta(\bm{D}) >  \varepsilon(|c_\eta(\bm{D})|, N, \delta)\Big\}$ converges to a limiting value as $\eta\rightarrow 0^+$ by appealing to dominated convergence.

\item[(iii)] We evaluate the limit directly and show it \emph{equals} $\delta$. We do so by showing that the \emph{first} and \emph{second relations} appearing in the proof of \Cref{prop:CG} hold with \emph{equality} as $\eta \to 0^+$. This is the heart of the proof, and the key reason why we provide our own proof of \Cref{prop:CG}. Thus, the same argument as in the proof of \Cref{prop:CG}, applied now with equalities in place of inequalities, yields equality throughout, so that for the complement event $$\mathbb{P}\Big\{\phi_\eta(\bm{D}) \leq \varepsilon(|c_\eta(\bm{D})|, 
N, \delta)\Big\} \to 1-\delta\,\,.$$

\end{enumerate}

\medskip
\noindent \textbf{Part~(i).} Let $U$ be a finite multiset, with order statistics $u_{(1)} \leq \cdots \leq u_{(|U|)}$, breaking ties randomly. Let $V$ be a multiset satisfying $c_\eta(U) \subseteq V \subseteq U$. Write $v_{(1)} \leq \cdots \leq v_{(|V|)}$ for the order statistics of $V$, again breaking ties randomly. We will show that $c_\eta(V) = c_\eta(U)$.
\medskip

To this end, denote $k^* = |c_\eta(U)|$. Observe that, by construction, $v_{(j)} = u_{(j)}$ for $j = 1, \dots, k^*$, and that $v_{(k^*+1)} \geq u_{(k^*+1)}$ since all elements in $V\setminus c_\eta(U)$ are no-smaller than $u_{(k^*+1)}$. 
\medskip

We now consider three cases for $k^*$:

\begin{enumerate}

\item[(i)] Case of $k^*=|V|$. Since $v_{(j)} = u_{(j)}$ for $j = 1, \dots, k^*$ and $u_{(j)}\le t_{j-1}$ for all $j = 1, \dots, k^*$, it must be $v_{(j)}\le t_{j-1}$ for all $j = 1, \dots, k^*$. Thus  $$c_\eta(V) = \{v_{(1)}, \dots, v_{(k^*)}\}=c_\eta(U)\,\, .$$

\item[(ii)] Case of $k^*<|V|$ and $k^*=N$. Because $v_{(j)} = u_{(j)}$ for $j = 1, \dots, k^*$ and $u_{(j)}\le t_{j-1}$ for all $j = 1, \dots, k^*$, it follows that $v_{(j)}\le t_{j-1}$ for all $j = 1, \dots, k^*$. Hence,  $$c_\eta(V) = \{v_{(1)}, \dots, v_{(k^*)}\}=c_\eta(U)\,\, .$$

\item[(iii)] Case of $k^*<|V|$ and $k^*<N$. It follows that $u_{(j)} \le  t_{j-1}$ for $j=1,\dots,k^*$,  while $u_{(k^*+1)} > t_{k^*}$. Since $v_{(j)} = u_{(j)}$ for $j = 1, \dots, k^*$ and $v_{(k^*+1)} \geq u_{(k^*+1)} > t_{k^*}$, we have $$c_\eta(V) =\{v_{(1)},\dots,v_{(k^*)}\}=\{u_{(1)},\dots,u_{(k^*)}\} =c_\eta(U)\,\, .$$

\end{enumerate}
Thus we conclude that, for any choice of \red{$\eta\ge 0$}, $c_\eta$ is a preferent compression.
\medskip

\noindent \textbf{Part~(ii).} 
We now consider the probability of change of compression associated with the dataset $\bf{D}$ obtained by sampling $N$ i.i.d. random variables uniformly in $[0,\,1]$ and with the compression $c_\eta$ 
$$
\phi_\eta(D) = \mb{P}\big\{c_\eta(c_\eta(D) \cup \{\bm{z}\}) \neq c_\eta(D) \;|\; D\big\} \, \, . 
$$
We aim to exactly quantify 
$$
F(\eta):= \pp(\phi_\eta(\bm{D}) > \varepsilon(|c_\eta(\bm{D})|, \, N,\, \delta) )\, \, .
$$
Observe that since $\mb{P}$ is the uniform distribution on $[0,\,1]$, the probability of encountering two or more identical values in the realization $D$ is zero. As such, $F(\eta)$ can be computed by considering only realizations for which no ties appear. 
\medskip

\red{
For notational convenience, denote $\varepsilon_k=\varepsilon(k,N,\delta)$ for $k=0,\dots,N$, and recall the thresholds $t_k=\varepsilon_k+\eta$ of Construction~\ref{con:main}. We will use that $t_k$ is non-decreasing in $k$, which follows from the monotonicity of $\varepsilon_k$.\footnote{\red{To see this, let $x:=1-\varepsilon$, so that \eqref{eq:calibration} reads
$h_k(x)=\delta/N$, where
\[
h_k(x)
=
\left (\sum_{m=k}^{N-1}
\frac{\binom{m}{k}}{\binom{N}{k}}x^{m-N}
\right)^{-1}.
\]
It is immediate to see that each $h_k$ is strictly increasing and that $h_{k+1}(x)>h_k(x)$ for all $x\in(0,1)$.  Since $h_k(1-\varepsilon_k)=\delta/N$ and $h_{k+1}(1-\varepsilon_{k+1})=\delta/N$, we have $1-\varepsilon_{k+1}<1-\varepsilon_k$, and therefore $\varepsilon_{k+1}>\varepsilon_k$.
}}

If $|c_\eta(D)|=N$, then $\phi_\eta(D)\leq\varepsilon_N=1$. Suppose instead that $k^*_\eta:=|c_\eta(D)|\leq N-1$. We claim that
\[ 
c_\eta(c_\eta(D)\cup\{z\})\neq c_\eta(D)
\quad\Longleftrightarrow\quad
z\leq t_{k^*_\eta}.
\]
Indeed, as $\mb{P}$ is uniform, $z\notin c_\eta(D)$ almost surely, and $c_\eta(D)\cup\{z\}$ has $k^*_\eta+1$ elements, so $z$ is inserted at some position $j\leq k^*_\eta+1$. 

Suppose $z\leq t_{k^*_\eta}$. The elements preceding $z$ are $u_{(1)},\dots,u_{(j-1)}$, which pass their tests by Construction~\ref{con:main}. 
The element $z$ also passes its test. Indeed, if $j\leq k^*_\eta$, then
$
z<u_{(j)}\leq t_{j-1},
$
while if $j=k^*_\eta+1$, then $z\leq t_{k^*_\eta}=t_{j-1}$.
Hence $z$ belongs to $c_\eta(c_\eta(D)\cup\{z\})$ but not to $c_\eta(D)$, and the compression
changes. 

Conversely, if $z>t_{k^*_\eta}$, monotonicity gives $u_{(k^*_\eta)}\leq t_{k^*_\eta-1}\leq t_{k^*_\eta}<z$ (with the first inequality omitted when $k^*_\eta=0$), so $z$ is appended after all compressed elements: the first $k^*_\eta$ tests are unaffected and the test at position $k^*_\eta+1$ fails. Thus $c_\eta(c_\eta(D)\cup\{z\})=c_\eta(D)$ and the compression remains unchanged.}
Therefore, 
\[
\begin{aligned}
\phi_\eta(D)
&= \pp( \bm{z} \leq  \varepsilon_{{k^*_\eta}} + \eta \, |\, D) = \min\{\varepsilon_{k^*_\eta} + \eta, 1\}\, \, ,
\end{aligned}
\]
where the first equality follows from the above claim and the definition of $\phi_\eta$, and the second since $\bm{z}\sim\mb{P}$ is uniformly distributed on $[0,1]$.

\noindent In particular, it follows that  $\phi_\eta(D) = \min\{\varepsilon_{k^*_\eta} + \eta, 1\} > \varepsilon_{k^*_\eta}$, for any choice of $\eta >0$ and $k^*_\eta <N$. We may therefore conclude that, for all $\eta>0$
\[
F(\eta)=\pp(\phi_\eta(\bm{D}) > \varepsilon(|c_\eta(\bm{D})|, \, N,\, \delta) ) = \pp(k^*_\eta < N) \, \, .
\]

\medskip
\noindent 
We now proceed to evaluate $\pp( \bm{k}^*_\eta < N)$:

\beq 
\begin{aligned}
\pp( \bm{k}^*_\eta < N) &= 1 - \pp(\bm{k}^*_\eta = N) \\
&= 1- \pp(\bm{z}_{(i)} \leq \varepsilon_{i-1} + \eta \text{ for all } i =1,\dots,N)  \\
&= 1- \ee \left[ \prod_{i=1} ^N I\{ \bm{z}_{(i)} \leq \varepsilon_{i-1} + \eta \} \right] \, \, .
\end{aligned}
\label{F.eta.defn}
\eeq

\noindent Note that, as $\eta \to 0^+$, the integrand converges; indeed,

$$
\prod_{i=1} ^N I\{ \bm{z}_{(i)} \leq \varepsilon_{i-1} + \eta \}  \to \prod_{i=1} ^N I\{ \bm{z}_{(i)} \leq \varepsilon_{i-1} \} \, \,.
$$

\noindent Since the integrand is bounded, it follows that, by dominated convergence, as $\eta \to 0^+$,
$$
F(\eta) = \pp( \bm{k}^*_\eta < N)
\to 
\pp( \bm{k}^*_0 < N)  %
\, \, . 
$$ 

\noindent Therefore, to evaluate the limit of $F(\eta)$ as $\eta \to 0^+$, it suffices to evaluate $\pp( \bm{k}^*_\eta < N)$ when $\eta=0.$
\begin{remark}
This does \emph{not} mean that the proof could have been carried out by forgoing the sequence of compressions $c_\eta$ and by focusing on $c_0$ only. In fact $c_0$ is such that $\phi_0(D)=\varepsilon_{k^*}$ when $k^*\le N-1$, which implies $F(0)=0$. Thus, $c_0$ is such that  
$\mb{P}\Big\{ \phi_0(\bm{D}) \leq  \varepsilon(|c_0(\bm{D})|, N, \delta)\Big\}=1$, which does \emph{not} establish tightness of the bound, as instead required.    
\end{remark}

\medskip

\noindent \textbf{Part~(iii).} We now proceed to evaluate $\pp( \bm{k}^*_0 < N)$. To this end, consider $m \leq N$, and let $D_m = \{z_1, \, \dots,\, z_m\}$, with $D_0=\emptyset$. For each $0 \leq k \leq m \leq N$, define
\[
v_{k, \, m} := \pp( c_0(\bm{D}_m) = \bm{D}_k) \,\,.
\]
This quantity represents the probability that compressing the dataset $D_m=\{z_1,\dots,z_m\}$ we obtain a dataset containing precisely its first $k$ samples, that is $D_k=\{z_1,\dots,z_k\}$. We can therefore utilize $v_{k,m}$ to decompose $\pp( \bm{k}^*_0 < N)$:
\beq
\begin{split}
\pp( \bm{k}^*_0 < N)
&= \sum_{k = 0} ^{N-1} \pp(|c_0(\bm{D}_N)|=k)\\
&= \sum_{k = 0} ^{N-1} \binom{N}{k} \pp( c_0(\bm{D}_N) = \bm{D}_k)\\
&= \sum_{k = 0} ^{N-1} \binom{N}{k} v_{k,\, N}\,\,,
\end{split}
\label{eq:p-k0-less-N}
\eeq
where the second line holds as there are $\binom{N}{k}$ subsets with $k$ elements and the probability that $c_0$ produces each such set of size $k$ equals the probability of producing $D_k$, due to the i.i.d. assumption and the permutation \mbox{invariance of  $c_0$.}

\medskip
We now show that the quantities $\{v_{k,m} : 0 \leq k \leq m \leq N\}$ satisfy two key relations, which mirror those governing $p_{k,m}$ in the proof of \Cref{prop:CG}, but now with equality. 

\medskip 
\noindent\textbf{First relation (with equality).} For any $0 \leq m \leq N$, $$\sum_{k=0}^{m}\binom{m}{k}v_{k,m} = 1\,\,.$$ This holds because the left hand side equals $\pp(\bm{k}^*_0 \leq m)$, and $\pp(\bm{k}^*_0 \leq m)=1$ since the compression of any dataset of size $m$ has cardinality at most $m$. 

\medskip \noindent\textbf{Second relation (with equality).} For all $0 \leq k \leq m \leq N-1$, $$v_{k,m+1} = (1-\varepsilon_k)\,v_{k,m}\,\,.$$ 
\red{
This holds because $c_0$ is preferent by Part~(i). In particular, since $\bm{D}_k\subseteq\bm{D}_m\subseteq\bm{D}_{m+1}$, the same argument as in Step~1 of the Appendix gives $\{c_0(\bm{D}_{m+1})=\bm{D}_k\}
=
\{c_0(\bm{D}_m)=\bm{D}_k\}
\cap
\{c_0(\bm{D}_{m+1})=c_0(\bm{D}_m)\}$.
Moreover, on the event $\{c_0(\bm{D}_m)=\bm{D}_k\}$, the argument in Part~(ii), applied with $\eta=0$, shows that $c_0(\bm{D}_{m+1})=c_0(\bm{D}_m)$ if and only if $\bm{z}_{m+1}>\varepsilon_k$, recalling that $t_k=\varepsilon_k$ when $\eta=0$. Since $\bm{z}_{m+1}$ is independent of $\bm{z}_1,\dots,\bm{z}_m$ and uniformly distributed on $[0,1]$, this occurs with probability $1-\varepsilon_k$. Hence,  $v_{k,m+1}=(1-\varepsilon_k)v_{k,m}$. Applying this recursion $N-m$ times yields
\beq
\label{eq:second-relation-equality-N-m}
v_{k,N} = (1-\varepsilon_k)^{N-m}v_{k,m}\,\,.
\eeq}

We can now conclude the proof by applying an identical argument and choice of $\alpha_{k,m}$ as in the proof of \Cref{prop:CG}. Specifically, choose 
\[ 
\alpha_{k,m} = \frac{\delta}{N}\cdot\frac{\binom{m}{k}}{\binom{N}{k}}(1-\varepsilon_k)^{m-N}\,\,, 
\] 
and note that, thanks to the calibration equation defining $\varepsilon_k$ in~\eqref{eq:calibration}, 
\[
\sum_{m=k}^{N-1}\alpha_{k,m} = \frac{\delta}{N}\cdot\frac{1}{\binom{N}{k}}\sum_{m=k}^{N-1}\binom{m}{k}(1-\varepsilon_k)^{m-N} = 1\,\,. 
\]

\noindent Therefore, taking linear combinations of \eqref{eq:second-relation-equality-N-m} with coefficients $\alpha_{k,m}$ yields
\beq
\label{eq:v-k-n-final}
v_{k,N} 
= \sum_{m=k}^{N-1} \alpha_{k,m} (1-\varepsilon_k)^{N-m}v_{k,m}
= \frac{\delta}{N}\sum_{m=k}^{N-1} \frac{\binom{m}{k}}{\binom{N}{k}} v_{k,m}\,\,.
\eeq
Substituting into \eqref{eq:p-k0-less-N} gives
\[ 
\begin{split}
\pp(\bm{k}^*_0 < N) 
&= \sum_{k=0}^{N-1}\binom{N}{k}v_{k,N}\\
&= \frac{\delta}{N}\sum_{k=0}^{N-1}\sum_{m=k}^{N-1}\binom{m}{k}v_{k,m}\\ 
&= \frac{\delta}{N}\sum_{m=0}^{N-1}\sum_{k=0}^{m}\binom{m}{k}v_{k,m}\\ 
&= \frac{\delta}{N}\sum_{m=0}^{N-1}1 = \delta\,\,, \end{split} \]
where we have used \eqref{eq:v-k-n-final} in the second line, swapped order of summation in the third line, and used the first relation in the fourth line. This concludes the proof.

\begin{ack}
Dario Paccagnan was partially supported by the EPSRC grant
EP/Y001001/1, funded by the International Science Partnerships Fund
(ISPF) and UKRI, and by an Imperial--MIT seed fund.
\end{ack}

\bibliographystyle{plainnat}
{\bibliography{references}}

\appendix 
\label{sec:appendix}
\section{Proof of second relationship}
\noindent
We aim to prove that $p_{k,m+1}\le(1-\varepsilon_k) p_{k,m}$. We divide the proof in three steps.
\medskip

\noindent
\textbf{Step 1.} Recall that 
\[
p_{k,m+1}   \doteq \mathbb P\{c({\bf D}_{m+1})={\bf D}_k~\text{and}~\phi({\bf D}_{m+1})> \varepsilon_k\}\,\,.
\]
Observe that, by preference, if dataset $D_{m+1}$ is such that $c(D_{m+1})=D_k$, $k\le m$, then it must be $c(D_{m+1})=c(D_{m})$ as $D_k\subseteq D_{m} \subseteq D_{m+1}$. Thus, $c({{D}}_{m+1})={{D}}_k$ holds if and only if $c({{D}}_{m+1})=c({{D}}_{m})~\text{and}~c({{D}}_{m})={{D}}_k$. Additionally $\phi(D_{m+1})=\phi(D_m)$ since one can replace $c(D_{m+1})$ with $c(D_m)$ in the definition of $\phi(D_{m+1})$. 
Thus,
\[
\begin{split}
p_{k,m+1}   =\mathbb P\{c({\bf D}_{m+1}&)= c({\bf D}_{m})~\text{and}~\\
&c({\bf D}_{m})={\bf D}_k~\text{and}~\phi({\bf D}_{m})> \varepsilon_k\}\,\,.
\end{split}
\]

\noindent \textbf{Step 2.} Note that the right-hand side above is the probability of the intersection of two events:
\begin{itemize}[leftmargin=8mm]
\item[$(A)$:] $c({\bf D}_{m+1}) = c({\bf D}_m)$, depending on $D_m$ and  $z_{m+1}$,
\item[$(B)$:] $c({\bf D}_m) = {\bf D}_k$ and $\phi({\bf D}_m) > \varepsilon_k$, depending on $D_m$.
\end{itemize}
\medskip
\noindent 
By tower property of the expectation:
\[
p_{k,m+1} = \mathbb{E}\big[\,\bm{1}_A \cdot \bm{1}_B\,\big] = \mathbb{E}\big[\,\mathbb{E}[\,\bm{1}_A \cdot \bm{1}_B \mid {\bf D}_m\,]\,\big]\,\,.
\]
Since $\bm{1}_B$ is fully determined by $\bm{D}_m$, it can be taken outside the inner expectation (which is taken over $\bm{z}_{m+1}$), giving
\beq
\label{eq:tower-prop}
p_{k,m+1} = \mathbb{E}\big[\,\mathbb{P}(A \mid {\bf D}_m) \cdot \bm{1}_B\,\big]\,\,.
\eeq
In the following step, we will show that for any fixed $D_m$ satisfying $B$, it is $\mathbb{P}(A \mid {D}_m)  < 1-\varepsilon_k$. Applying this to \eqref{eq:tower-prop} will complete the proof since it gives
\[
p_{k,m+1} 
 \leq (1-\varepsilon_k) \cdot \mathbb{E}\big[\bm{1}_B\big]
 = (1-\varepsilon_k) p_{k,m}\,\,.
\]

\noindent
\textbf{Step 3.} We now show that for any fixed $D_m$ satisfying $B$
\[
\mathbb{P}(A \mid {D}_m)\le 1-\phi(D_m) < 1-\varepsilon_k\,\,.
\]
First, note that the last inequality follows trivially by definition of $B$, since $1-\phi(D_m) < 1- \varepsilon_k$. Thus it remains to show the former. To see this, fix $D_m$, and note that %
\beq
\{c(D_m \cup \{z_{m+1}\}) = c(D_m)\} \subseteq
\{c(c(D_m) \cup \{z_{m+1}\}) = c(D_m)\}\,\,.
\label{eq:set-containment}
\eeq
Indeed, take any $z_{m+1}$ in the left-hand set, so that $c(D_m \cup \{z_{m+1}\}) = c(D_m)$. By preference, since $c(D_m) \cup \{z_{m+1}\}$ is contained between its superset $D_m \cup \{z_{m+1}\}$, as $c(D_m) \subseteq D_m$, and its subset $c(D_m \cup \{z_{m+1}\}) = c(D_m)$, it must be $c(c(D_m) \cup \{z_{m+1}\}) = c(D_m)$, which places $z_{m+1}$ in the right-hand set. Taking probabilities conditional on $D_m$ in \eqref{eq:set-containment}, and recalling the definition of $\mathbb{P}(A \mid {D}_m)$ and $\phi(D_m)$ gives the desired result. 

\blue{
\section{Removing the assumption of non-concentrated mass}
We now remove the assumption of non-concentrated mass appearing in \eqref{eq:non-concentrated-mass} and used in \Cref{subsec:simplifiedproof}. For this purpose we follow the same argument of~\cite[page 42]{campi2023compression}, which we report here with our notation purely for completeness.
\medskip

To that end, we augment each random element $\bm{z}_i$ with a random variable
$\bm{\theta}_i$ uniformly distributed over $[0,1]$ and independent of
$\bm{z}_i$, so as to form an i.i.d. sequence of augmented samples
$\bm{z}'_1 = (\bm{z}_1,\bm{\theta}_1), \bm{z}'_2 = (\bm{z}_2,\bm{\theta}_2), \dots$
taking values in $\mc{Z}' = \mc{Z}\times[0,1]$. %
Since $\bm{\theta}_i$ is uniform and independent of $\bm{z}_i$, the augmented
samples are non-concentrated by construction, so that
\beq
\label{eq:no-mass-augmented}
\mb{P}\big\{\bm{z}'_i = z'\big\} = 0\,, \quad \forall z' \in \mc{Z}'\,\,.
\eeq

Given a multiset $U'$ of augmented samples, let $\pi[U']$ denote the multiset obtained by extracting the first components. We define a compression $c'$, to be applied to multisets of augmented samples, as the compression for which $\pi[c'(U')] = c(\pi[U'])$ and, among sub-multisets of $U'$ whose projection is $c(\pi[U'])$, $c'$ selects augmented samples with lowest second component.
\medskip

{We next show that $c'$ also satisfies the preference property.  Let $T' = c'(U')$ and let $V'$ be such that $T' \subseteq V' \subseteq U'$. Projecting, $c(\pi[U']) = \pi[T'] \subseteq \pi[V'] \subseteq \pi[U']$, so preference of $c$ gives $c(\pi[V']) = \pi[T']$. Fix now $z \in \mc{Z}$ and let $j$ be its multiplicity in $\pi[T']$. By definition of $c'(U')$, the $j$ augmented samples in $T'$ with first component $z$ are those with the $j$ lowest second components among all such samples in $U'$. Since $T' \subseteq V' \subseteq U'$, they are also those with the $j$ lowest second components in $V'$. Hence $c'(V') = T' = c'(U')$, thus establishing the preference property of $c'$.}

\medskip

We are thus in a position to apply to $c'$ the proof that has been
developed before under the assumption of non-concentrated mass. Indeed, \eqref{eq:no-mass-augmented} holds by construction, while $c'$ has just been shown to be preferent. Letting
$\phi'(\bm{D}') := \mb{P}\{c'(c'(\bm{D}') \cup \{\bm{z}'\}) \neq c'(\bm{D}') \mid \bm{D}'\}$, where $\bm{D}' := \{\bm{z}'_1,\dots,\bm{z}'_N\}$, we have
\[
\mb{P}\big\{\phi'(\bm{D}') > \varepsilon(|c'(\bm{D}')|,N,\delta)\big\} \leq \delta\,\,.
\]
However, for any $D'$ with $\pi[D'] = D$ and any $z' = (z,\theta)$, $c(c(D)\cup\{z\}) \neq c(D)$ implies that
$
c'(c'(D') \cup \{z'\}) \neq c'(D')
$, which gives $\phi(\bm{D}) \leq \phi'(\bm{D}')$ almost surely. Moreover, $|c(\bm{D})| = |c'(\bm{D}')|$. We thus obtain the desired result
\[
\mb{P}\big\{\phi(\bm{D}) > \varepsilon(|c(\bm{D})|,N,\delta)\big\}
\leq
\mb{P}\big\{\phi'(\bm{D}') > \varepsilon(|c'(\bm{D}')|,N,\delta)\big\}
\leq \delta\,\,.
\]
}

\end{document}